%% file: main.tex
\documentclass[runningheads]{llncs}
\usepackage[T1]{fontenc}

\usepackage{graphicx}
\usepackage{mathabbrev}
\usepackage{macros}
\usepackage{amssymb}
\usepackage{amsmath}
\usepackage{multirow}
\usepackage{pgfplots}
    \pgfplotsset{compat = newest}
\usepackage{subcaption}
\usepackage{todonotes}
\usepackage{thm-restate}
\usepackage{algorithm}
\usepackage{algorithmicx}
\usepackage[noend]{algpseudocode}

\usepackage{ulem}
\usepackage{makecell}

\begin{document}
\title{Term Rewriting Based on Set Automaton Matching}
%
%
\author{Mark Bouwman\and Rick Erkens}

%
\authorrunning{M. Bouwman, R. Erkens}
\authorrunning{}
%
\institute{}
\institute{Eindhoven University of Technology
\\ The Netherlands
\\ \email{\{m.s.bouwman, r.j.a.erkens\}@tue.nl}
}

\maketitle              
\begin{abstract}
In this article we investigate how a subterm pattern matching algorithm can be exploited to implement efficient term
rewriting procedures.
From the left-hand sides of the rewrite system we construct a set automaton,
which can be used to find all redexes in a term efficiently.
We formally describe a procedure that,
given a rewrite strategy, interleaves pattern matching steps and rewriting steps and thus
smoothly integrates redex discovery and subterm replacement.
We then present an efficient implementation that instantiates this procedure with outermost rewriting,
and present the results of some experiments.
Our implementation shows to be competitive with comparable tools.
\end{abstract}
\input{content}
\subsubsection*{Acknowledgements}
We want to thank
Hubert Garavel for sending us the benchmark set used in Rewrite Engine Competition.

\bibliographystyle{splncs04}
\bibliography{bibliography}
\end{document}

%% file: content.tex
\section{Introduction}
The motivation for this research is to develop a new rewriter
for fast state space generation and model checking in mCRL2 \cite{toolsetpaper}.
In this formalism, every state is associated with a vector of normalized data expressions.
Over the years more complex term rewrite systems have been generated from
industrial domain-specific languages.
Typically these contain many complex rules,
which makes an optimized rewrite engine desirable.
In our setting, rules may have arbitrary overlap between left-hand sides,
as opposed to functional programs and orthogonal rewrite systems.
Moreover the left-hand sides can be non-linear (i.e. variables can occur multiple times),
and the rewrite strategy is not fixed.

To allow one to deal with large rewrite systems,
implementations of term rewriting use efficient pattern matching algorithms to find redexes.
They exploit the overlap between the left-hand sides of the rules,
and mostly use efficient \emph{root pattern matching algorithms},
i.e.~techniques that preprocess the left-hand sides into a data structure that can be used to
quickly look for matches at a specific position in the term.
Root pattern matching can be applied to every position in a term to yield every redex,
but this solution is not efficient in general.
Function symbols that are inspected in one call to the matching algorithm,
will be inspected again when the same matching algorithm is called at a later point from a different root.
This means that symbols will be visited many times,
resulting in duplicate work.
A theoretically better solution is to use an algorithm that scans a term
for every pattern at every position simultaneously.
Such \emph{subterm pattern matching algorithms} exist,
but they are generally not used in practical implementations of term rewriting.

We set out to investigate whether an efficient rewrite engine can be made
by using a subterm pattern matching algorithm.
We are not aware of any tools or methods that explicitly and formally intertwine these two,
even though the idea of using subterm pattern matching for rewriting
dates back to the 1980s \cite{hoffmann:interpreter,hoffmann:matching}.
The contribution of this paper is a formal description of a rewrite procedure that
explicitly interleaves rewrite steps with
subterm pattern matching steps,
based on the set automaton pattern matching technique described in \cite{setautomaton}.
We prove that our procedure correctly implements term rewriting.
Lastly we describe a prototype implementation,
and compare its performance to other rewrite engines by
using the benchmark set created by Garavel et al. in the context of the recurring Rewrite Engine
Competition \cite{garavel:rec}.
Even though it is a prototype,
its performance is competitive with the fastest algebraic rewriters.

The set automaton matching algorithm generalizes the Aho-Corasick string matching algorithm \cite{aho:stringmatching} to
sets of term patterns.
The main advantage of set automata is that they inspect every symbol of a term only once
by traversing the term from top to bottom.
Even though our rewrite procedure is based on this algorithm,
it does not call the pattern matching function as a complex subroutine.
It rather interleaves automaton-guided pattern matching steps (single function symbol observations) and rewrite steps,
and saves the matching information in a data structure called a \emph{configuration tree}.
Based on the available matching information,
a strategy decides to continue searching for redexes (and where)
or to rewrite using a previously found redex (and which one).
This method allows us to preserve a part of the matching work when we apply a rewrite step.

\subsection{Related work}
The idea to use subterm matching for term rewriting and equational reasoning dates back to Hoffmann and O'Donnell
\cite{hoffmann:interpreter,hoffmann:matching}.
In \cite{dewit:reductiemachine} some algorithms are described based on their ideas.
Huet and L\'evy use matching dags for orthogonal rewrite systems in \cite{huet:computationsII}.
Fokkink et al. transform left-linear rewrite systems into so-called minimal rewrite systems with only shallow patterns and
right-hand sides \cite{fokkink:arm}.
The transformation on the right-hand sides of the rules
is based on one of Hoffmann and O'Donnell's algorithms \cite{hoffmann:matching}.

Current state-of-the-art rewrite engines use efficient root pattern matching techniques.
Maude, the fastest rewrite engine among the algebraic tools of the latest edition of the Rewrite Engine Competition \cite{garavel:rec},
deploys automaton-based solutions to several kinds of pattern matching problems \cite{maude}.
In particular it uses discrimination nets \cite{christian:discrimination} for matching linear patterns.
Dedukti uses decision trees \cite{hondet:dedukti} to deal with non-linearity and higher-order
pattern matching.
The compiling rewriter of mCRL2 uses match trees \cite{weerdenburg:matchtrees} that compute subtitutions
on the fly as well.
An overview of root pattern matching techniques is given in \cite{har:termindexing}.

In functional languages like Haskell \cite{haskell} and Clean \cite{clean}
algebraic data types are specified with constructor function symbols
that cannot be the head symbol of any rewrite rules.
In combination with other language choices, such as priorities between patterns,
this ensures that pattern matching is easier than in general term rewriting.
Since our target language is algebraic,
we make no distinction between constructors and function-defining symbols,
and impose no restrictions like orthogonality.

For term rewriting, a deterministic top-down solution to subterm matching is desirable,
so we opted to use set automata for our procedures.
The other deterministic top-down algorithms in \cite{cleophas:related,hoffmann:matching}
are reductions to string matching:
they collect stringpath matches during the matching operation which must be merged while rewriting.
Set automata do not suffer this administrative overhead and output the term matches themselves while matching.
The drawback of set automata however, is that they require a look-ahead of at most the largest depth of any left-hand side.
This look-ahead is necessary since top-down deterministic tree automata
are not powerful enough to perform any pattern matching \cite{tata}.
In some cases this look-ahead can be mitigated,
but a general solution to use set automata for linear time matching is an open problem.

\subsection{Structure of the paper}
In Section~\ref{sec:preliminaries} we present some preliminary definitions.
In Section~\ref{sec:setautomaton} we recall the set automaton matching technique from \cite{setautomaton}.
A set automaton can be executed to find all redexes in a term,
by storing intermediate configurations in a \emph{set}.
For the rewriting procedures that we introduce,
it is important to save the matching information in a more structured way than a set.
To this end, we introduce \emph{configuration trees},
where the configurations that would previously be stored in a set,
appear at the leaves of the tree.

To further aid the reader in understanding set automata,
we recall in detail how a set automaton is constructed from a set of patterns
in Section~\ref{sec:construction}.
This section is largely copied from \cite{setautomaton}.

In Section~\ref{sec:pruningrewriter} we formally describe a rewrite procedure that takes an arbitrary strategy,
and searches for redexes while maintaining a configuration tree.
When a rewrite step is performed,
a part of the information in the configuration tree is saved.
We show that the rewrite procedure correctly implements term rewriting, regardless of the strategy.

Since set automata can only deal with linear patterns,
we define a rewrite procedure that supports non-linear rewrite systems in Section~\ref{sec:nonlinear}.
A slight adaptation to the set automaton construction that requires non-linear patterns to be yielded in a separate set,
suffices to support these rewrite systems,
albeit in a (arguably necessarily) non-optimal way.

To bridge the gap between formal presentation and practice
we discuss an efficient implementation of our rewrite procedure in Section~\ref{sec:implementation}.
Specifically it uses an outmost rewrite strategy and uses a stack to store the configuration tree.
We also present some benchmarks and extensively compare the results to those of some other rewrite engines.

Lastly in Section~\ref{sec:conclusion}
we present some concluding remarks and some ideas for
further improvements of the technique that we intend to explore in the future.

\section{Preliminaries}\label{sec:preliminaries}
We recall some preliminaries, found in e.g. \cite{termrewriting}.
A \emph{ranked alphabet} is a finite set of function symbols $\bF$
paired with an arity mapping $\#:\bF\to\bN$.
The set of \emph{terms over $\bF$ with variables in $\bV$} is denoted by $\bT(\bF,\bV)$.
The set of variables of $t$ is given by $\vars(t)$.
A ground term is a term without variables.

A \emph{position} is a list of positive integers.
We use $\bP$ to denote the set of all positions
and we use $\epsilon$ to denote the empty list,
also referred to as the \emph{root position}.
The concatenation of two positions $p$ and $q$ is denoted by $p.q$.
The root position acts as a unit element with respect to concatenation.
Given a set $X$ we often denote a pair $(x,p)$ in $X\times\bP$ by $x@p$.
The \emph{term domain function} $\cD:\bT(\bF,\bV)\to\cP(\bP)$ maps a term to the set of positions
on which it has function symbols or variables,
for example $\cD(f(g(a),b))=\{\epsilon,1,1.1,2\}$.
Similarly the \emph{edge function} $\cE:\bT(\bF,\bV)\to\cP(\bP)$ maps a term $t$ to the subset of $\cD(t)$
where $t$ has variables.
Define $\cD_{\setminus\cE}(t)=\cD(t)\setminus\cE(t)$.
Given a term $t$ and a position $p\in\cD(t)$,
the subterm of $t$ at position $p$ is denoted by $t|_p$.
Replacing the subterm of $t$ at position $p$ by $u$ is denoted by $t[u]_p$.
A \emph{pattern} is a term, but not a variable.
Said pattern is \emph{linear} if no variable occurs more than once in $\ell$.
Denote the \emph{head symbol} of a pattern $\ell$ by $\hd(\ell)$.

A \emph{substitution} is a mapping $\sigma:\bV\to\bT(\bF,\bV)$;
the term obtained by replacing all variables $x\in\vars(t)$ by $\sigma(x)$ is denoted by $t^\sigma$.
Term $t$ \emph{matches} pattern $\ell$ iff $t=\ell^\sigma$ for some $\sigma$.
A \emph{rewrite rule} is a pattern $\ell$ paired with a term $r$, denoted $\ell\to r$,
such that all variables of $r$ occur in $\ell$ as well.
A \emph{term rewrite system (TRS)} is a finite, non-empty set of rewrite rules.
A \emph{redex of $t$} is a rule $\ell\to r$ and a position $p$ such that
$t|_p$ matches $\ell$.
Such a redex is often denoted by $(\ell\to r)@p$.

Let $t$ be a term with redex $(\ell\to r)@p$.
Then $t$ rewrites to $t'$,
denoted by $t\step{(\ell\to r)@p} t'$,
iff there is a term $u$ with $p\in\cD(u)$ such that
$t=u[\ell^\sigma]_p$ and $t'=u[r^\sigma]_p$.
We write $t\rr t'$ if $t$ rewrites to $t'$ by some rule on some position.
The transitive reflexive closure of $\rr$ is denoted by $\rr^*$.
A term $t$ is in \emph{normal form} iff $t$ has no redexes.
Given a rule $\ell\to r$ and a term $t$ such that $\ell$ matches $t$ at position $p$,
we denote by $t[(\ell\to r)@p]$ the unique term $t'$ such that $t\step{(\ell\to r)@p}t'$.

\section{Set automata}\label{sec:setautomaton}
Since set automata form the foundation for the rewrite procedures,
we recall them in detail from \cite{setautomaton}.
First we recall their formal definition.
Then we give two methods to execute them on terms.
In the next section we recall the set automaton construction algorithm.

Set automata generalize Aho and Corasick's string matching automata \cite{aho:stringmatching} to terms.
They prescribe how to search for patterns in a term in a top-down fashion and allow us to match all patterns
in every subterm simultaneously.
To find all matches, we have to inspect every symbol of a term only once,
modulo some look-ahead.

Let $\cR$ be a TRS with $\cL=\{\ell\mid\ell\to r\in\cR\}$ being its set of left-hand sides.
A set automaton for $\cL$ over the alphabet $\bF$
is a tuple $(S,s_0,L,\delta,\out)$ where
\begin{itemize}
\item $S$ is a finite set of \emph{states};
\item $s_0\in S$ is the \emph{initial state};
\item $L:S\to\bP$ is a \emph{state labelling function};
\item $\delta:S\times\bF\to\cP(S\times\bP)$ is a \emph{transition function};
\item $\out:S\times\bF\to\cP(\cL\times\bP)$ is an \emph{output function}.
\end{itemize}

\begin{figure}[h!]
\centering
\includegraphics[scale=0.8]{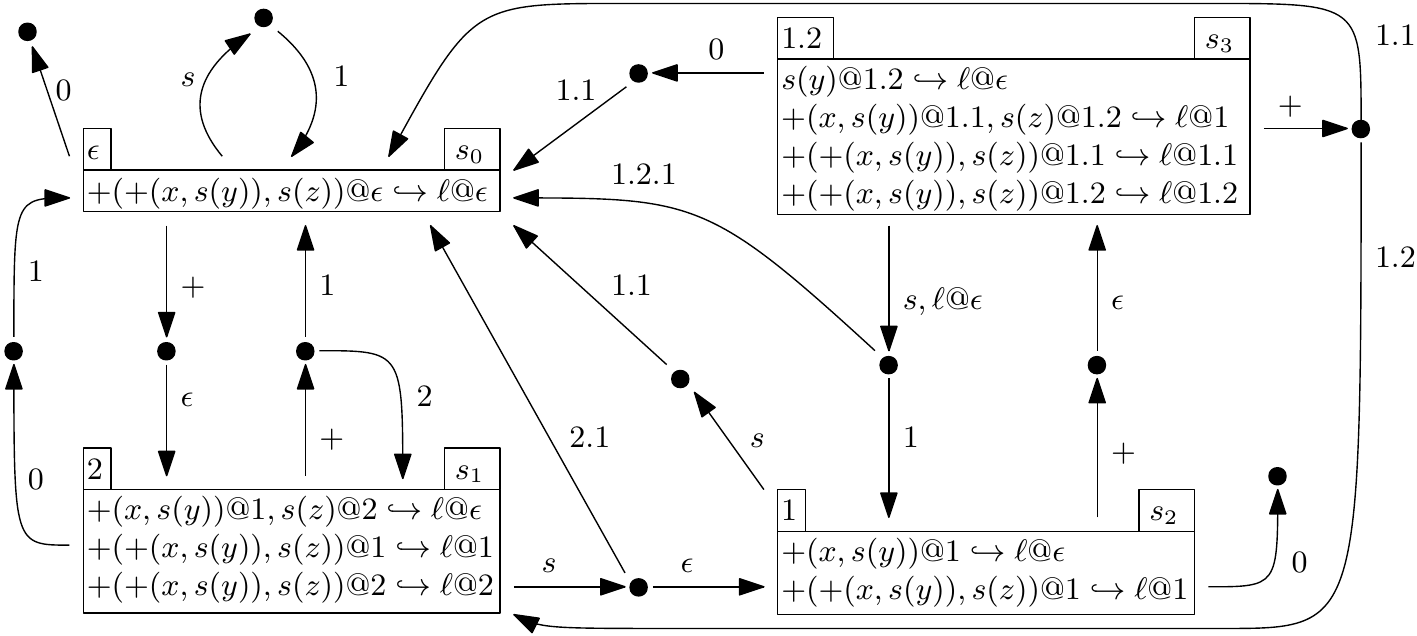}
\caption{A set automaton for $\ell=\add(\add(x,\s(y)),\s(z))$.}
\label{fig:smallboy}
\end{figure}

Consider the pattern $\ell=\add(\add(x,\s(y)),\s(z))$.
In Fig.~\ref{fig:smallboy} the set automaton for the singleton pattern set $\{\ell\}$
is graphically displayed.
The lines of the form $x \hookrightarrow y$ in the boxes are important in the next section where we discuss the
construction, but for now they can be ignored.
The states are named $s_0,\dots,s_3$, displayed on the top right corner.
The state labels are given on the top left corner; they provide information on what position to inspect next.
The drawn hypertransitions are a combination of the transition function and the output function,
e.g. the $s$-labelled transition from state $s_3$ is captured by
$\delta(s_3, s) = \{(s_0, 1.2.1), (s_2, 1)\}$ and $\out(s_3, s) = (\ell, \epsilon)$.

Note that the transition and output functions are total: for every state and symbol combination there is a transition.
Though for some transitions there may not be target states: e.g. $\delta(s_0, 0) = \emptyset$.
Intuitively, matching goes as follows. We start in the initial state and with the root position:
the \emph{configuration} $(s_0, \epsilon)$.
The state label instructs us what position to inspect relative to the position in the configuration.
With the inspected symbol we select the corresponding transition.
If the transition has an output then we found a match.
We follow the transition function, yielding new states to visit (possibly multiple, possibly none).
If the transition specifies a non-root position then the position in the configuration is deepened.
For example, suppose we want to match on the term $s(0)$.
Then we start with the initial configuration $(s_0, \epsilon)$.
We inspect  $\hd(s(0)|_{\epsilon.\epsilon}) = s$ and take the selfloop transition,
yielding the configuration $(s_0, 1)$.
We now inspect $\hd(s(0)|_{1.\epsilon}) = 0$, and take the corresponding transition.
This does not lead to new configurations, hence matching is finished.

Let us formalize this matching procedure. To find redexes in a ground term $t$, we traverse it
by using the set automaton.
To this end,
a set of configurations\footnote{Hence the name \emph{set} automaton}
- pairs $(s,p)$ consisting of a state $s$ in the set automaton
and a position $p$ in the term - needs to be maintained.
From a configuration $(s,p)$ we observe the function symbol of term $t$ on position $p.L(s)$: $\hd(t|_{p.L(s)})$.
Then we compute the successor configurations and the output as follows.
\begin{itemize}
\item
For all successors $(s',q)\in\delta(s,f)$ obtained from configuration $(s,p)$
we add a new configuration $(s',p.q)$ to the configuration set.
\item
When $(\ell,q)\in\out(s,f)$ is given as output from configuration $(s,p)$
then there is a match for $\ell$ at position $p.q$.
In all of our examples the position $q$ is just $\epsilon$,
but when a pattern is a proper subpattern of another,
then $q$ may be deeper.
\end{itemize}
The set of all redexes of a ground term $t$ is retrieved by $\eval(s_0,\epsilon,t)$ where
\[
\eval(s,p,t)=\{(\ell,p.q)\mid (\ell,q)\in\out(s,f)\}\cup\bigcup_{(s',p')\in\delta(s,f)} \eval(s',p.p',t|_{p'})
\]
where $f=\hd(t|_{L(s)})$.

\begin{example}\label{ex:matching}
For a more elaborate example let us consider the ground term $t = +(+(+(0,s(0)),s(0)),s(0))$. The table below shows how the set automaton is used to find all matches in $t$ (whilst inspecting every symbol once).
\vspace*{0.5cm}

\centering
\begin{tabular}{l|l|p{2cm}|p{2cm}|l}
Configuration set & Configuration & Observed \newline position & Observed \newline symbol & Output \\\hline
$(s_0,\epsilon)$ & $(s_0,\epsilon)$ & $\epsilon$ & $+$ & \\
$(s_1,\epsilon)$& $(s_1,\epsilon)$ & $2$ & $s$ & \\
$(s_2,\epsilon)$, $(s_0,2.1)$& $(s_2,\epsilon)$ & $1$ & $+$ & \\
$(s_0,2.1)$, $(s_3, \epsilon)$  & $(s_0, 2.1)$ & $2.1$ & $0$ & \\
$(s_3, \epsilon)$ & $(s_3, \epsilon)$ & $1.2$ & $s$ & $\ell@\epsilon$\\
$(s_0,1.2.1)$, $(s_2,1)$ & $(s_0,1.2.1)$ & $1.2.1$ & $0$ & \\
$(s_2,1)$ & $(s_2,1)$ & $1.1$ & $+$ &\\
$(s_3,1)$ & $(s_3,1)$ & $1.1.2$ & $s$ & $\ell@1$\\
$(s_0,1.1.2.1)$, $(s_2,1.1)$ & $(s_0,1.1.2.1)$ & $1.1.2.1$ & $0$ & \\
$(s_2,1.1)$ & $(s_2,1.1)$ & $1.1.1$ & $0$ & \\
\end{tabular}
\end{example}

\subsection{Configuration trees}
Instead of storing configurations in a set we can also represent the discovered configurations
in a tree structure with $(s_0, \epsilon)$ as root. 
Originally this structure was used in the correctness proof of set automata,
but it has appeared to be very useful for rewriting,
as we can now also reason about partially completed configuration trees.
An unexplored configuration appears as a $\bud$-labelled pair $\bud(s,p)$,
a childless vertex in the tree which we refer to as a \emph{bud}.
Upon observing a function symbol in the subject term,
this bud is `grown' into an $\node$-labelled triple $\node(s,p,cts)$ called a \emph{node}.
The component $cts$ of a node is a possibly empty set of successor trees;
when the node is created by growing a bud,
it consists of buds,
one for every successor state given by $\delta(s,f)$.
If there are no successor states given by $\delta(s,f)$ then $cts$ is empty.
Hence the trees $\bud(s,p)$ and $\node(s,p,\emptyset)$ differ in the sense that
the bud indicates an unexplored configuration and the node indicates an explored configuration
with no successors.
Formally the set of \emph{configuration trees} $\cC\cT$ is formed by the grammar
$\cC\cT ::= \bud(s,p) \mid \node(s,p,cts)$ where $cts$ is a finite, possibly empty set of configuration trees.
Define the set of $\bud$/$\node$-labelled configurations of $ct$ by $\buds(ct)$ and $\nodes(ct)$ respectively.
The set of all configurations is given by $\confs(ct)$.
%
The \emph{completed configuration tree} for term $t$ is defined by $\completed(s_0,\epsilon,t)$ where
\[
\completed(s,p,t)= \node(s,p,\{\completed(s',p.p',t)\mid (s',p')\in\delta(s,\hd(t|_{p.L(s)}))\})
\,.
\]

\begin{example}
The configuration tree in Fig.~\ref{fig:configtree} is associated to line 5 of the table in Example \ref{ex:matching}.
Note that the configuration tree does not have the same shape as $t$.
\end{example}
\begin{figure}[h!]
\centering
\includegraphics[scale=0.8]{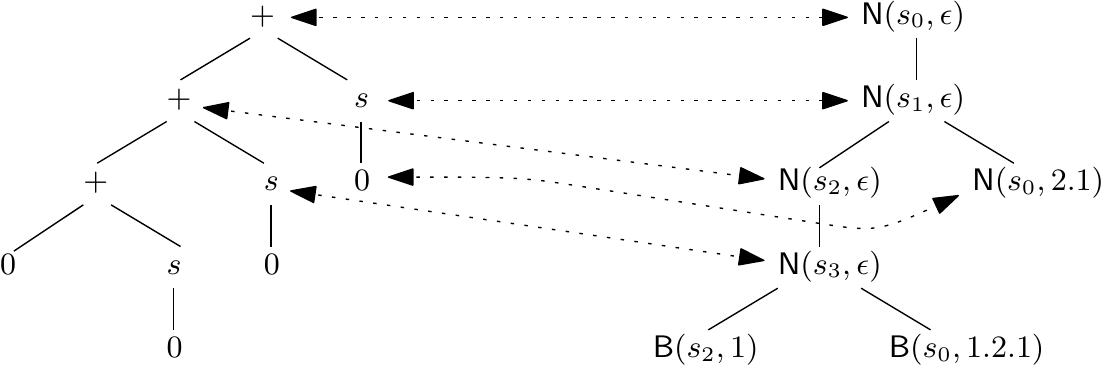}
\caption{The term $t = +(+(+(0,s(0)),s(0)),s(0))$ and a partial configuration tree.}
\label{fig:configtree}
\end{figure}
Define the abbreviation $\completed(t)=\completed(s_0,\epsilon,t)$.
A completed configuration tree is a one-to-one correspondence between
configurations and the function symbols of a term,
assigned to it by a set automaton.
\begin{theorem}\label{thm:completedCT}
Let $M=(S,s_0,L,\delta,\out)$ be a set automaton generated from the TRS $\cR$.
Then $\node(s,p,cts)\mapsto p.L(s)$ is a bijection from
the subtrees of $\completed(t)$ to $\cD(t)$.
\end{theorem}
This correspondence suffices to find every redex in a term as follows.
In a configuration $(s,p)$ the symbol $f=\hd(t|_{p.L(s)})$ yields some matches
as prescribed by $\out(s,f)$.
Define the discovered matches of a configuration $(s,p)$ and $t$ by
\[
\matches(s,p,t) = \{(\ell@p.q) \mid (\ell,q)\in\out(s,\hd(t|_{p.L(s)}))\}
\,.
\]
The following theorem states the correctness claim for pattern matching.
\begin{theorem}\label{thm:correctnessofmatching}
For every ground term $t$, the set of all redexes of $t$ is given by
$\bigcup\{\matches(s,p,t) \mid (s,p)\in\nodes(\completed(t))\}$.
\end{theorem}

\section{Construction of a set automaton}\label{sec:construction}
We recall from \cite{setautomaton} how a set automaton is constructed
from a pattern set.
The examples and the explanation overlap largely with the original work.
To understand the rewrite procedures it is not necessary to understand the construction.

We create the states and the transition function on the fly.
To this end we first give a state encoding that represents its information in the pattern matching process.
The informal explanation follows the formal one.

The \emph{subpatterns of a pattern $\ell$} are given by
$\sub(\ell)=\{\ell|_p\mid p\in\cD_{\setminus\cE}(\ell)\}$.
This function is extended to sets of patterns $\cL$ by $\sub(\cL)=\bigcup_{\ell\in\cL}\sub(\ell)$.
The sets of \emph{match obligations} $\MO$ and \emph{match announcements} $\MA$ are defined by
\[
\MO=\cP(\sub(\cL)\times\bP)\setminus\{\emptyset\}
\qquad
\MA=\cL\times\bP
\,.
\]
A \emph{match goal} is a match obligation paired with a match announcement,
i.e.\ a pair in $\MO\times\MA$.
To limit the amount of parentheses, we denote a match goal
$(\{(\ell_1,p_1),\dots,(\ell_n,p_n)\},(\ell@p))$ by $\ell_1@p_1,\dots,\ell_n@p_n\hookrightarrow\ell@p$.
Such a match goal should be read as: ``in order to announce a match for pattern $\ell$ at position $p$,
it remains to observe the (sub)pattern $\ell_i$ on position $p_i$, for all $1\leq i\leq n$''.
We identify states with non-empty sets of match goals, i.e.\ $S=\cP(\MO\times\MA)\setminus\{\emptyset\}$.
This encoding formally captures the idea that a state carries every pattern that still needs to be seen
whenever we arrive in this state in the matching process.

Recall the set automaton from Fig.~\ref{fig:smallboy}.
The initial state $s_0$ has one match goal,
namely $\add(\add(x,\s(y)),\s(z))@\epsilon \hookrightarrow \add(\add(x,\s(y)),\s(z))@\epsilon$.
The match announcement of this goal is abbreviated to $\ell@\epsilon$ in the figure.
Intuitively, whenever we are matching from $s_0$,
we still need to observe every pattern from the root.

State $s_1$ can be reached by following the $\add$-transition from the initial state,
and it has three match goals.
The first goal $\add(x,\s(y))@1,\s(z)@2\hookrightarrow\ell@\epsilon$ is obtained by reducing the match obligation
of $s_0$ to $\add(x,\s(y))@1,\s(z)@2$.
The top symbol $\add$ has already been inspected;
therefore the subpatterns $\add(x,\s(y))$ and $\s(z)$ should still be observed at their respective positions
in order to yield a pattern match for $\ell$ at the root position.
Furthermore there are the goals
$\add(\add(x,\s(y)),\s(z))@1 \hookrightarrow \add(\add(x,\s(y)),\s(z))@1$,
and $\add(\add(x,\s(y)),\s(z))@2 \hookrightarrow \add(\add(x,\s(y)),\s(z))@2$.
These are considered \emph{fresh} goals on a lower position.
Since the purpose of a set automaton is to yield all pattern matches at every position,
these fresh goals are added throughout the construction.

We give two kinds of match goals a special status.
A match goal of the form $\ell@p\hookrightarrow\ell@p$ is called \emph{fresh}.
It represents that we aim to find a pattern match for $\ell$ on position $p$,
but we have not observed any function symbols of that pattern yet.
A match goal of the form $mo\hookrightarrow\ell@\epsilon$ is called a \emph{root goal}.

\subsection{Initial state}
Consider a non-empty pattern set $\cL$.
We start the construction of a set automaton for $\cL$ from the initial state,
and compute the transition function and the other states on the fly.
From the initial state all patterns still need to be observed from the root position.
In terms of the state encoding,
the initial state is formally defined by
$s_0=\{\ell@\epsilon\hookrightarrow \ell@\epsilon\mid\ell\in\cL\}$.
In other words,
the initial state contains every fresh root goal.

\subsection{State labels}
For every state $s$ there must be a position label $L(s)$ in order to construct the transitions from $s$.
It can be chosen from the match obligation positions,
and we demand the extra constraint that this position should be part of a root match goal.
The construction guarantees that every state has a root goal,
which is shown in \cite{setautomaton}.
Similar to Adaptive Pattern Matching Automata~\cite{sekar:adaptive},
there might be multiple positions available to choose from.
Any of such positions can be chosen in the construction of the automaton,
but this position needs to be fixed when $s$ is created.
For example, the initial state only has root goals, so it is labelled with the empty position,
formally $L(s_0)=\epsilon$.

\subsection{Transitions}
From a state $s$ we define $\delta(s,f)$ for every function symbol $f$
in three steps.
First, we compute the so-called $f$-derivative of $s$.
This covers the semantics of subterm pattern matching completely.
That is,
by only taking derivatives we obtain a correct matching automaton.
However, this automaton will be infinite since the derivative operation adds more match goals
with longer positions.

To mitigate this problem we perform two more steps after computing the derivative.
To counteract the growing amount of match goals,
we partition them into independent equivalence classes.
In Fig.~\ref{fig:smallboy} this appears as an arrow going out of a black dot for every equivalence class.
To counteract the growing positions
we maximally shorten all positions in every equivalence class.
In Fig.~\ref{fig:smallboy} the positions on the arrows
going out of the black dots,
indicate which part of the positions was removed.
The shortened equivalence classes of the derivative are then new states,
from which the process of constructing the automaton proceeds.
Eventually one only encounters states that have already been processed.

\subsubsection{Derivatives.}
To capture the semantics of subterm pattern matching we introduce $f$-derivatives of states.
This terminology is borrowed from Brzozowski derivatives of regular expressions~\cite{brzozowski:derivatives}.
Given a state $s$ with state label $L(s)=p$, and a symbol $f$,
we determine $f$-derivative of $s$ in three steps.
Firstly, we look at the match goals of $s$ which have a match obligation with position $p$.
These obligations should be \emph{reduced}.
Secondly, we take all match goals of $s$ which do not have a match obligation with position $p$.
These goals remain \emph{unchanged}.
Thirdly, we add \emph{fresh goals} for every pattern and every $i$ with $1\leq i\leq \#f$.

Consider a match goal $\ell_1@p_1,\dots,\ell_n@p_n\hookrightarrow \ell@p$,
and suppose that we need to reduce this goal modulo position $p_i$ with function symbol $f$.
There are three cases that need to be considered.
\begin{itemize}
\item If $i=1$, $n=1$ and $\ell_1=f(x_1,\dots,x_n)$ for variables $x_1,\dots,x_n$
then $f@p_i$ is the last observation needed to complete the match goal.
Then $\ell@p$ is recorded in the output function of the automaton and the match goal disappears.
\item If $\hd(\ell_i)\neq f$,
the observation $f@p_i$ leads to the conclusion that this pattern cannot match.
Then the match obligation is discarded.
\item Otherwise $\ell_i=f(t_1,\dots,t_{\#f})$ but the match goal is not completed yet.
Then the other pairs $\ell_j@p_j$ with $i\neq j$ remain unchanged,
the pair $\ell_i@p_i$ is removed from the obligation and
for all $j$ such that $t_j\notin\bV$ we add the pair $t_j@p_i.j$ to the obligation of this goal.
\end{itemize}
Let $\pos(mo)$ be the set of positions in a match obligation $mo$.
Formally, reducing a match obligation is captured by $\reduce:\MO\times\bF\times\bP\to\MO\cup\{\emptyset\}$
defined by:
\begin{align*}
\reduce(mo,f,p)=\, &
    \{\ell@q\in mo\mid q\neq p\} \ \cup \\
    & \{\ell|_i@p.i\mid\ell@p\in mo\wedge 1\leq i\leq\#f\wedge\ell|_i\notin\bV\}
    \,.
\end{align*}
Using the mapping $\reduce$, define the $f$-derivative of state $s$ by
\begin{align*}
\deriv(s,f)&=\mathit{reduced} \cup \mathit{unchanged} \cup \mathit{fresh}\,,\text{where} \\
\mathit{reduced}&=
    \{\reduce(mo,f,L(s))\hookrightarrow ma\mid mo\to ma\in s \ \wedge \\
     & \hspace{2cm} \exists\ell:\ell @L(s)\in mo \wedge \hd(\ell)=f
     \wedge \reduce(mo,f,L(s))\neq\emptyset\} \\
\mathit{unchanged}&=\{mo\hookrightarrow ma\in s\mid L(s)\notin\pos(mo)\} \\
\mathit{fresh}&=\{\ell@L(s).i\hookrightarrow \ell@L(s).i\mid\ell\in\cL\wedge 1\leq i\leq\#f\}
\,.
\end{align*}

\begin{example}\label{ex:smallboyderivative}
Consider state $s_1$ of the set automaton in Fig.~\ref{fig:smallboy}.
The parts of $\deriv(s_1,\s)$ are computed as follows:
\begin{align*}
\mathit{unchanged}&=\{
\add(\add(x,\s(y)),\s(z))@1\hookrightarrow \ell@1
\} \\
\mathit{reduced}&=\{
\add(x,\s(y))@1\hookrightarrow\ell@\epsilon
\} \\
\mathit{fresh}&=
\{\add(\add(x,\s(y)),\s(z))@2.1\hookrightarrow\ell@2.1
\}\,.
\end{align*}
Note that the goal $\add(\add(x,\s(y)),\s(z))@2\hookrightarrow\ell@2$ disappears completely
since the expected symbol $\add$ mismatches with the considered symbol $\s$ at position $2$.
\end{example}

\subsubsection{Derivative partitioning.}
One application of $\deriv$ creates new match goals with longer positions.
Repeated application of $\deriv$ therefore results in an automaton with an infinite amount of states.
We first mitigate the issue of new match goals by partitioning the computed derivative with respect to
a dependency relation on match goals.

Note from Example~\ref{ex:smallboyderivative} that the derivative has two match obligations at position
$1$, and one match obligation at position $2.1$.
To obtain an efficient matching algorithm,
it is important that goals with overlapping positions end up in the same equivalence class.
When observing a function symbol at this relative position $1$,
both goals should be reduced simultaneously.
Conversely, sets of goals that are independent from each other can be separated to form a new state with fewer match goals.
When evaluating a set automaton this creates the possibility of exploring parts of the subject term independently.
Given a state $s$,
define the \emph{direct dependency relation} $R_{dep}$ for all goals $g_1=mo_1\hookrightarrow \ell_1@p_1$
and $g_2=mo_2\hookrightarrow\ell_2@p_2$ by
$g_1\mathrel{R_{dep}} g_2$ if, and only if, $\pos(mo_1)\cap\pos(mo_2)\neq\emptyset$.
Note that $R_{dep}$ is reflexive, since match obligations cannot be empty,
and $R_{dep}$ is also symmetric.
But $R_{dep}$ is not transitive, since for the obligations
$mo_1=\{t_1@1\}$,$mo_2=\{t_1@1, t_2@2\}$ and $mo_3=\{t_2@2\}$
we have $\pos(mo_1)\cap\pos(mo_2)=\{1\}$ and $\pos(mo_2)\cap\pos(mo_3)=\{2\}$,
but $\pos(mo_1)\cap\pos(mo_3)=\emptyset$.
To obtain an equivalence we simply take the transitive closure ${\dep}=R_{dep}^*$,
which we refer to as (indirect) dependency.

We partition $\deriv(s,f)$ with respect to $\dep$.
Each equivalence class then corresponds to a new state.
The set of equivalence classes of the derivative is denoted by $[\deriv(s,f)]_{\dep}$.
We use $K$ to range over equivalence classes.

\begin{example}\label{ex:smallboypartition}
Consider the $\s$-derivative of Example~\ref{ex:smallboyderivative}.
Partitioning
yields
\begin{align*}
K_1&=
\{
\add(\add(x,\s(y)),\s(z))@1\hookrightarrow \ell@1, \quad
\add(x,\s(y))@1\hookrightarrow\ell@\epsilon
\}\\
K_2&=
\{
\add(\add(x,\s(y)),\s(z))@2.1\hookrightarrow\ell@2.1
\}
\end{align*}
\end{example}

\begin{example}\label{ex:smallboypartition2}
Consider the $\add$-derivative of $s_2$, which is exactly $s_3$.
Note that the goals $\s(y)@1.2\hookrightarrow\ell@\epsilon$ and
$\add(\add(x,\s(y)),\s(z))@1.1\hookrightarrow\ell@1.1$ are not directly dependent,
but the goal $\add(x,\s(y))@1.1,\s(z)@1.2\hookrightarrow\ell@1$ is directly dependent to both goals.
Therefore we obtain a singleton partition.
\end{example}

\subsubsection{Lifting the positions of classes.}
After partitioning, we shorten the positions of every equivalence class
to mitigate the problem of growing positions.
Let $\pos_{\MA}(K)$ denote the positions of the match announcements of $K$.
We want to `lift' every position in every goal of $K$ by the greatest common prefix of $\pos_{\MA}(K)$.
To ease the notation we write $\gcp(K)$ instead of $\gcp(\pos_{\MA}(K))$.
Since all positions in a state are of the form $\gcp(K).p'$,
we can replace them by $p'$.
Define $\lift(s)$ by $\lift(s)=\{(\lift(mo),\ell@p')\mid (mo,\ell@\gcp(s).p')\in s\}$
where $\lift(mo)=\{\ell@p'\mid \ell@\gcp(s).p'\in mo\}$.

For a state $s$ and a function symbol $f$,
we fix $\delta(s,f)=\{(\lift(K),\gcp(K))\mid K\in[\deriv(s,f)]_{\dep}\}$.
Note that $\gcp(K)$ is also recorded in each transition since it tells us how to traverse the term.

\begin{example}
Continuing in Example \ref{ex:smallboypartition},
we compute the greatest common prefix and corresponding transition for the two equivalence classes.
For $K_1$ we have $\gcp(K_1)=\gcp(\{1,\epsilon\})=\epsilon$.
Then $\lift(K_1)=K_1=s_2$, and therefore $(s_2,\epsilon)\in\delta(s_1,\s)$.
Class $K_2$ has one goal with $\gcp(K_2)=\gcp(\{2.1\})=2.1$.
Then $\lift(K_2)=\{\add(\add(x,\s(y)),\s(z))@\epsilon\hookrightarrow\ell@\epsilon\}$,
which yields the transition $(s_0,2.1)\in\delta(s_1,\s)$.
\end{example}

\subsection{Output patterns}\label{sec:output}
The output patterns after an $f$-transition
are simply the match announcements that accompany the match obligations
that reduce to $\emptyset$:
\[\out(s,f)=\{ma\in \MA\mid f(x_1,\dots,x_n)@L(s)\hookrightarrow ma\in s \wedge x_1,\dots,x_n\in\bV\}\,.\]

\begin{example}
Consider state $s_3$ in Fig.~\ref{fig:smallboy}.
The goal $\s(y)@1.2\hookrightarrow\ell@\epsilon$ can be completed upon observing $\s$ at position $1.2$,
so we fix $\out(s_3,\s)=\{\ell@\epsilon\}$.
\end{example}

\subsection{Construction summary}
The following definitions summarize the construction of a set automaton.
\begin{itemize}
\item $S$ is the set of all reachable states;

\item $s_0=\{\ell@\epsilon\hookrightarrow \ell@\epsilon\mid\ell\in\cL\}$;

\item $L(s)\in\pos(mo)$ for some root goal $mo\hookrightarrow ma\in s$;

\item $\delta(s,f)=\{(\lift(K),\gcp(K))\mid K\in[\deriv(s.f)]_{\dep}\}$ for all $(s,f)\in S\times\bF$; and

\item $\out(s,f)=\{ma\in \MA\mid f(x_1,\dots,x_n)@L(s)\hookrightarrow ma\in s\}$ for all $(s,f)\in S\times\bF$.
\end{itemize}
For the remainder of this paper we assume that instead of patterns, the output function yields rewrite rules.
That is, given the computed output function $\out$ and the TRS $\cR$ we use
$\out'(s,f)=\{(\ell\to r)@p\mid \ell\to r\in\cR\wedge \ell@p\in\out(s,f)\}$.

\section{Rewriting with configuration trees}\label{sec:pruningrewriter}
We formally define a rewrite procedure that uses set automaton matching
to explore the completed configuration tree step by step.
This procedure only works for left-linear TRSs.
An extension for non-linear TRSs will be discussed in the next section.


Algorithm~\ref{alg:pruningrewriter} is the formal foundation for set automaton-based rewriting.
It maintains a set of found redexes, and a configuration tree.
Furthermore it uses some auxiliary manipulations on the configuration tree which we define below.
The strategy determines whether to reduce a found redex, or to continue matching a bud in $ct$.
Formally a strategy is a partial mapping $\select:\conftree\times\cP(\cR\times\bP)\rightharpoonup(S\times\bP)\cup(\cR\times\bP)$
such that $\select(ct,reds)\in\buds(ct)\cup reds$,
for all pairs $(ct,reds)$ such that $\buds(ct)\cup reds$ is non-empty.
Note that strategies are only partial since they are undefined on budless configuration trees
paired with empty redex sets.

If the strategy selects a bud $\bud(s,p)$,
we inspect the head symbol of term $t|_{p.L(s)}$,
and grow the configuration tree according to the transitions given by the set automaton.
Formally this growth is defined by
\begin{align*}
\grow&(ct,s,p,t) = \\
&
\begin{cases}
\node(s,p,\{\bud(s',p.p')\mid (s',p')\in\delta(s,\hd(t|_{p.L(s)}))\}) & \text{if $ct=\bud(s,p)$} \\
\bud(s',p') & \text{if $ct=\bud(s',p')$} \\
            & \quad \text{with $(s',p')\neq(s,p)$} \\
\node(s',p',\{\grow(ct',s,p,t)\mid ct'\in cts\}) & \text{if $ct=\node(s',p',cts)$}
\end{cases}
\end{align*}
Intuitively $\grow(ct,s,p,t)$ replaces every bud $\bud(s,p)$ by a node with the same configuration.
The successors of this node are new buds;
their configurations are given by the used set automaton transition.
When matching a part of a complete configuration tree,
the grown bud is unique due to Theorem~\ref{thm:completedCT}.
After growing the configuration tree, newly found redexes are added to the redex set according to the $\matches$ function
defined in Section~\ref{sec:setautomaton}.

If the strategy selects a redex $(\ell\to r)@p$,
then the term is rewritten to $t[(\ell\to r)@p]$.
The procedure then \emph{prunes} the configuration tree as follows.
Since a set automaton inspects the input term in a top-down fashion,
there must be a configuration in $ct$ where position $p$ was inspected.
We call this the \emph{initialization configuration for redex $(\ell\to r)@p$},
more precisely, the configuration $(s_i,p_i)$ such that $p_i.L(s_i)=p$.
Upon applying the rewrite step $t[(\ell\to r)@p]$,
only the subterm at position $p$ changes.
Any matching that was not done below position $p$ is still valid:
it suffices to prune the configuration tree up to the initialization configuration of the used redex.
Define $\prune(\bud(s,p),q)=\bud(s,p)$ and
\[
\prune(\node(s,p,cts), q)=
\begin{cases}
\bud(s,p) & \text{if $p.L(s)=q$} \\
\node(s,p,\{\prune(ct',q) \mid ct'\in cts\}) & \text{otherwise}\,.\\
\end{cases}
\]

Lastly, the set of redexes needs to be updated after a rewrite.
All redexes that were obtained in the pruned configuration tree are no longer valid.
Given $ct$, denote by $ct[p]$ the sub-configuration tree with root configuration $(s,q)$ in $ct$ such that $p=q.L(s)$.
Note that $ct[p]$ does not necessarily exist,
but if it exists and $ct$ is a fragment of a completed tree,
then $ct[p]$ is unique by Theorem~\ref{thm:completedCT}.

\begin{algorithm}[h!]
\caption{A rewrite procedure that remembers part of the matching information
by pruning a part of the configuration tree.
We assume the availability of a set automaton $(S,s_0,L,\delta,\out)$
constructed for the TRS at hand.
}
\label{alg:pruningrewriter}
\begin{algorithmic}[1]
    \Procedure{Normalize}{\text{Term} $t_0$, \text{Strategy} $\select$}
        \State $t\gets t_0$
        \State $reds\gets\emptyset$          \Comment{Redexes}
        \State $ct\gets\bud (s_0,\epsilon)$ \Comment{Configuration tree}
        \While{$\buds(ct)\neq\emptyset \vee reds \neq\emptyset$}
            \State $action\gets\select$ from $\buds(ct) \cup reds$
            \If{$action$ is a configuration $(s,p)$}
                \State $ct\gets\grow(ct,s,p,t)$
                \State $reds\gets reds\cup\matches(s,p,t)$
            \ElsIf{$action$ is a redex $(\ell\to r)@p\in reds$}
                \State $reds \gets
                    reds \setminus \bigcup\{\matches(s,q,t)\mid (s,q)\in\nodes(ct[p])\}$
                \State $ct\gets \prune(ct,p)$
                \State $t\gets t[(\ell\to r)@p]$
            \EndIf
        \EndWhile
        \Return $t$
    \EndProcedure
\end{algorithmic}
\end{algorithm}

\begin{example}
Consider the TRS $\cR_{if}$ given by the five rules
\begin{align*}
&R_1:if(true,x,y)\rr x  & R_2:if(false,x,y)\rr y & \\
&R_3:not(true)\rr false & R_4:not(false)\rr true & \qquad R_5:not(not(x))\rr x \,.
\end{align*}
In Fig.~\ref{fig:setautomatonif} there is a set automaton generated for $\cR_{if}$.
Contrary to the automaton in Fig.~\ref{fig:smallboy},
this simpler one only produces configuration trees that have the same shape as
the terms that it is executed on.
In Fig.~\ref{fig:examplepruning} there is the term
$t=if(not(not(true)),false,true)$ on the left, and in the middle
there is a partially explored configuration tree of $t$.
It is obtained by growing $\bud(s_0,\epsilon)$ with selected configurations $(s_0,\epsilon)$,
$(s_1,\epsilon)$ and $(s_2,1)$.

The last configuration yields a match for $R_5$ at position $1$.
A rewrite with this redex yields $t'=if(true,false,true)$
and the configuration tree on the right.
The symbol $if$ does not need to be inspected again,
since the initialisation configuration of $R_5@1$ is $(s_1,\epsilon)$.
\end{example}

\begin{figure}
\centering
\includegraphics[scale=0.8]{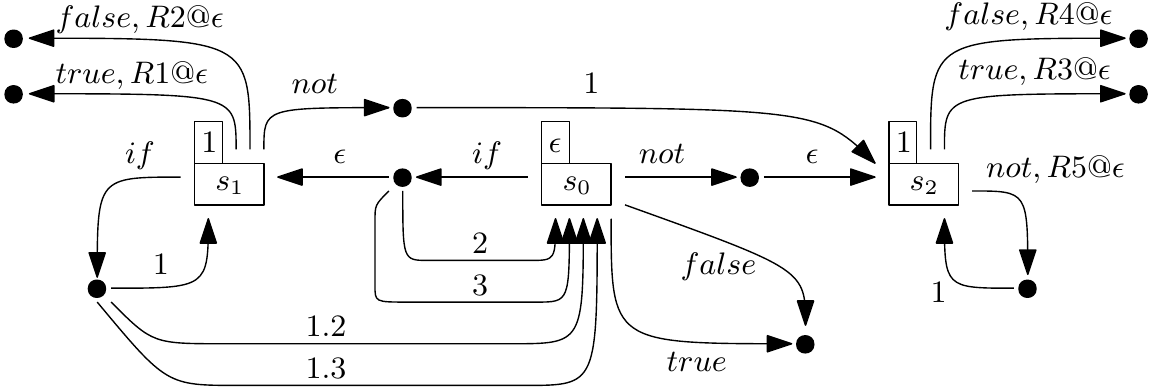}
\caption{A set automaton for $\cR_{if}$. To limit clutter,
$\delta(s_2,if)$ is not displayed.}
\label{fig:setautomatonif}
\end{figure}

\begin{figure}
\centering
\resizebox{\textwidth}{!} {
\includegraphics[scale=0.8]{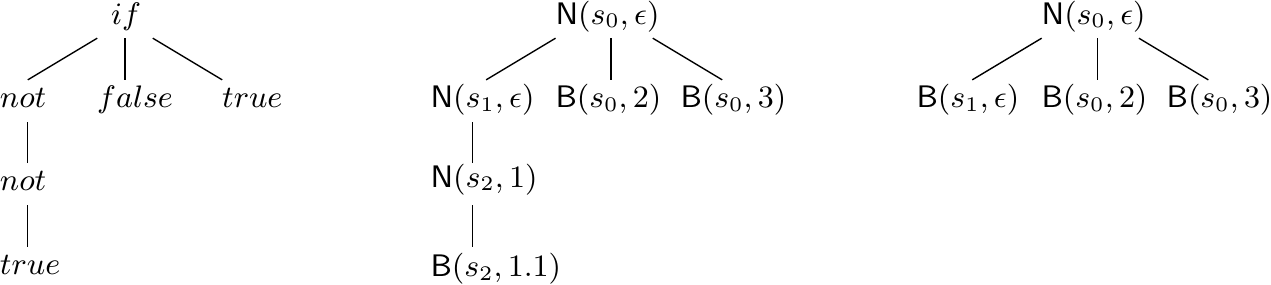}
}
\caption{
The term $t=if(not(not(true)),false,true)$,
a partial configuration tree leading to the discovery of $R_5@1$
and the result of pruning.
}
\label{fig:examplepruning}
\end{figure}

\subsection{Correctness}
We prove that Algorithm~\ref{alg:pruningrewriter},
satisfies a simple invariant.
Intuitively, we show that the maintained configuration tree remains a fragment of the completed configuration tree
of the current term, and the redex set correctly maintains all redexes
that can be derived from the observations in the current configuration tree.
First we need some formalities on configuration trees.

The `is fragment of' relation $\fragment$ on $\cC\cT$ is the smallest relation such that
\begin{itemize}
\item $\bud(s,p)\fragment\bud(s,p)$ and $\bud(s,p)\fragment\node(s,p,cts)$ and
\item $\node(s,p,cts)\fragment\node(s,p,cts')$ if and only if there is a bijection
$\varphi:cts\to cts'$ such that $ct\fragment\varphi(ct)$ for all $ct\in cts$.
\end{itemize}
This definition makes $\langle\cC\cT,\fragment\rangle$ a partial ordering.
Intuitively we have $ct_1\fragment ct_2$ whenever we can change $ct_1$ into
$ct_2$ by zero or more applications of $\grow$.
In particular we are interested in fragments of completed configuration trees.

\begin{restatable}{lemma}{pruneproperties}\label{lem:prune}
The following properties hold for operations on configuration trees.
\begin{itemize}
\item if $ct\fragment\completed(t)$ then $\grow(ct,s,p,t)\fragment\completed(t)$
for all $(s,p)\in\buds(\completed(t))$;
\item if $ct\fragment ct'$ then $\prune(ct,p)\fragment \prune(ct',p)$;
\item $\prune(ct,p)\fragment ct$;
\item if $t\step{(\ell\to r)@p} t'$,
then $\prune(\completed(t),p)=\prune(\completed(t'),p)$; and
\item if $ct[p]$ exists then $\nodes(\prune(ct,p))=\nodes(ct)\setminus\nodes(ct[p])$.
\end{itemize}
\end{restatable}

\begin{restatable}{lemma}{pruninginvariant}\label{lem:pruninginvariant}
Let $t_0$ be a ground term.
Then the invariant given by
\[
t_0\rr^* t
\wedge
ct\fragment\completed(t)
\wedge
reds=\bigcup\{\matches(s,p,t)\mid (s,p)\in\nodes(ct)\}
\,,
\]
is initialized and maintained by $\normalize$.
\end{restatable}
\begin{proof}
For initialization,
since $t=t_0$ we trivially have $t_0\rr^* t$.
Moreover $\bud(s_0,\epsilon)\fragment\completed(t)$
by definition of $\completed(t)$.
Lastly
since $\nodes(ct)=\nodes(\bud(s,p))=\emptyset$
we have
$reds=\emptyset=\bigcup\{\matches(s,p,t)\mid (s,p)\in\nodes(ct)\}$.

For invariant maintenance,
consider $t$, $ct$ and $reds$ such that they satisfy the invariant
and suppose that $\buds(ct)\cup reds\neq\emptyset$.
We prove that both cases of the if-statement preserve the invariant.

\begin{itemize}
\item Suppose that $\select(ct,reds)=(s,p)$, some configuration in $\buds(ct)$,
and let $f=\hd(t|_{p.L(s)})$.
Since $t$ is unchanged, $t_0\rr^* t$ holds by assumption.
Moreover $\grow(ct,s,p,t)\fragment\completed(t)$ by Lemma~\ref{lem:prune}.
Conclude by the calculation:
\begin{align*}
reds' &= reds \cup \matches(s,p,t) \\
      &= \bigcup\{\matches(s',p',t)\mid (s',p')\in\nodes(ct)\}
            \cup \matches(s,p,t) \\
      &= \bigcup\{\matches(s',p',t)\mid (s',p')\in\nodes(ct)\cup\{(s,p)\}\} \\
      &= \bigcup\{\matches(s',p',t)\mid (s',p')\in\nodes(\grow(ct,s,p,t))\}\,.
\end{align*}

\item Otherwise suppose that $\select(ct,reds)=(\ell\to r)@p$, some redex in $reds$.
Let $t'=t[(\ell\to r)@p]$.
By $ct\fragment\completed(t)$ and Theorem~\ref{thm:correctnessofmatching}
it follows that $(\ell\to r)@p$ is a redex of $t$.
Therefore $t_0\rr^* t\step{(\ell\to r)@p} t[(\ell\to r)@p]=t'$,
hence $t_0\rr^*t'$ as needed.
Next we show that $\prune(ct,p)\fragment\completed(t')$.
By assumption $ct\fragment\completed(t)$.
Then Lemma~\ref{lem:prune} enables the calculation:
\begin{align*}
\prune(ct,p)
&\fragment   \prune(\completed(t),p)  \\ 
&=           \prune(\completed(t'),p) \\  
&\fragment   \completed(t')            
\end{align*}
Hence $\prune(ct,p)\fragment\completed(t')$ as required.
Next, let
\[reds'=reds \setminus \bigcup\{\matches(s,q,t)\mid (s,q)\in\nodes(ct[p])\}
\,.
\]
Then:
\begin{align*}
reds' &= reds \setminus \bigcup\{\matches(s,q,t)\mid (s,q)\in\nodes(ct[p])\} \\
      &= \bigcup\{\matches(s,q,t)\mid (s,q)\in\nodes(ct)\} \\
        &\hspace{0.3cm}\setminus \bigcup\{\matches(s,q,t)\mid (s,q)\in\nodes(ct[p])\} \\
      &= \bigcup\{\matches(s,q,t)\mid (s,q)\in\nodes(ct)\setminus\nodes(ct[p])\} \\
      &= \bigcup\{\matches(s,q,t)\mid (s,q)\in\nodes(\prune(ct,p))\}\,.
\end{align*}
Therefore $reds'=\bigcup\{\matches(s,q,t)\mid (s,q)\in\nodes(\prune(ct,p))\}$ as needed.
\end{itemize}
\end{proof}

Lemma~\ref{lem:pruninginvariant} states that Algorithm~\ref{alg:pruningrewriter} produces valid rewrite steps.
Of course, termination depends on the TRS and the strategy,
which we have abstracted from so far and will remain to do so until Section~\ref{sec:implementation}.
We conclude with the corollary of partial correctness.

\begin{corollary}
If $\normalize(t_0,\select)$ returns $t$,
then $t$ is a normal form of $t_0$.
\end{corollary}
\begin{proof}
Suppose that the procedure terminates with values $t$, $ct$ and $reds$.
By Lemma~\ref{lem:pruninginvariant} the invariant is satisfied.
By termination, $\buds(ct)=\emptyset$ and $reds=\emptyset$.
Since $ct$ has no buds, $ct=\completed(t)$.
By Theorem~\ref{thm:correctnessofmatching} and $reds=\emptyset$ there are no redexes in $t$.
To conclude, $t_0\rr^* t$, hence $t$ is a normal form of $t_0$.
\end{proof}

\section{Supporting non-linear rewrite rules}\label{sec:nonlinear}
We extend Algorithm~\ref{alg:nonlinearrewriter} to support non-linear TRSs
using simple subterm comparisons.
Before we address the problem,
some remarks on the literature are in order.
It is well-known that non-linear patterns cannot be matched by a regular tree automaton \cite{tata}.
The proof for this claim is a textbook pumping lemma application.
To support these patterns, tree automata with equality and disequality constraints were introduced
\cite{bogaert:constraints,tata}.
These feature transitions that are guarded by finite constraints on subterms,
and traverse terms in a bottom-up fashion.
Little is known about the efficiency of this approach however.
Tree automata with constraints can check some equalities and disequalities multiple times
as the guards on the transitions can be arbitrarily large.

The interplay between linear pattern matching and subterm equality is also interesting in
the root pattern matching case.
Subterm (dis)equalities may rule out pattern matches, and vice versa,
function symbol observations may rule out the necessity for checking (dis)equality.
Similar automata/decision trees have been used in the root pattern matching case
featuring decision nodes with only a single equality check per transition
\cite{anpmajournalpaper,hondet:dedukti,weerdenburg:matchtrees}.
These approaches are susceptible to redundant equality checks as well;
an optimal solution is unknown to exist.

Originally, set automata do not support non-linear pattern matching.
Since patterns with repeating variables appear to be much-requested whereas an optimal solution is
a significant theoretical challenge,
we simply deploy \emph{consistency checks}.
Whenever a non-linear pattern is yielded by the set automaton,
we need the required subterm equalities to hold before it can be used as a redex.
This simple solution suffices to support non-linear matching,
but we need to take more care in adapting Algorithm~\ref{alg:pruningrewriter} for non-linear rewriting.

We proceed with the following formal definitions.
Let $\ell$ be a pattern,
and recall that $\cE(\ell)$ is the set of positions where $\ell$ has a variable.
Define $\pi(\ell)$
as the unique partition of $\cE(\ell)$ such that
for all sets $P\in\pi(\ell)$,
for all $p,q\in\cE(\ell)$
we have that $p,q\in P$ if, and only if $\ell|_p =\ell|_q$. 
In other words, $\pi(\ell)$ contains the sets of positions where $\ell$ has the same variable.
Note that linear patterns only have singleton partitions, and therefore
$|\pi(\ell)|=|\cE(\ell)|$.
Non-linear patterns have at least one set in their partition
of size at least two, and therefore $|\pi(\ell)|<|\cE(\ell)|$.
We say that \emph{$t$ is consistent with $\pi(\ell)$} iff
for all $P\in\pi(\ell)$, for all $p,q\in P$ we have
$t|_{p}=t|_{q}$.
Lastly, we say that \emph{$t$ pre-matches $\ell$} iff
for all $p\in\cD_{\setminus\cE}(\ell)$
we have that $\hd(t|_{p})=\hd(\ell|_{p})$.

\begin{proposition}\label{prop:nlmatching}
A term $t$ matches the pattern $\ell$ iff
$t$ pre-matches $\ell$ and $t$ is consistent with $\pi(\ell)$.
\end{proposition}

For linear patterns, pre-matching coincides with matching since
since the consistency check is satisfied vacuously.
A set automaton checks exactly pre-matching for all patterns and all subterms of $t$ simultaneously.
It does not check for consistency though.
We opted to outsource these checks to the rewrite procedure,
and have set automata yield the non-linear patterns in a separate set.

A \emph{non-linear set automaton} for $\cL$
is a structure $(S,s_0,L,\delta,\out_L,\out_{NL})$
where $S,s_0,L$ and $\delta$ are defined as in set automata,
and $\out_L,\out_{NL}:S\times\bF\to\cP(\cL)$ yield a set of linear patterns and
a set of non-linear patterns respectively.
Given a set automaton $(S,s_0,L,\delta,\out)$ for $\cL$ we obtain a non-linear set automaton
$(S,s_0,L,\delta,\out_{L},\out_{NL})$ for $\cL$ where
$\out_L(s,f)=\{\ell\in\out(s,f)\mid|\pi(\ell)|=|\cE(\ell)|\}$ and
$\out_{NL}(s,f)=\{\ell\in\out(s,f)\mid|\pi(\ell)|<|\cE(\ell)|\}$.
Define $\matches_L$ and $\matches_{NL}$ (for $X\in\{L,NL\}$) by
\[
\matches_X(s,p,t)=\{(\ell\to r)@p.q\mid (\ell\to r)@q\in\out_X(s,\hd(t|_{p.L(s)}))\}
\,.
\]

Algorithm~\ref{alg:nonlinearrewriter} rewrites a term by using a non-linear set automaton.
When applied to a linear TRS it behaves exactly the same as its linear counterpart.
In the case of non-linear patterns more bookkeeping comes into play.
\begin{itemize}
\item All non-linear patterns yielded by the set automaton are first stored in a set of ambiguous matches,
    for they need to be checked for consistency before they can be used to reduce the term.
\item We add the option to check an ambiguous match for consistency
    to the same pool of options as the choice of whether to match a symbol or reduce a redex.
    When a match is consistent then it is \emph{enabled}, and can be chosen in the future to reduce the term.
    Otherwise it is stored in a set of \emph{disabled} matches.
\item When a reduction is applied,
    we need to remove all redexes that were obtained below the pruning point as in Algorithm~\ref{alg:pruningrewriter}.
    Additionally, some disabled matches may be enabled by the change, and vice versa.
    This applies to all enabled and disabled matches that have a repeating variable
    above the position where the rewrite step was applied.
    The procedure moves those matches to the set of ambiguous redexes,
    so they can be checked again at some later point.
\end{itemize}

\begin{algorithm}[!h]
\caption{
A rewriting procedure that supports non-linear patterns.
We assume the availability of a non-linear set automaton $(S,s_0,L,\delta,\out_L,\out_{NL})$
constructed for the TRS at hand.
}
\label{alg:nonlinearrewriter}
\begin{algorithmic}[1]
    \Procedure{NormalizeNonLinear}{\text{Term} $t_0$, \text{Strategy} $\select$}
        \State $t\gets t_0$ 
        \State $reds_L\gets\emptyset$   \Comment{Linear redexes}
        \State $am, dis, en\gets\emptyset$ \Comment{Ambiguous, disabled and enabled pre-matches}
        \State $ct\gets\bud (s_0,\epsilon)$ \Comment{Configuration tree}
        \While{$\buds(ct)\neq\emptyset \vee reds_L \neq\emptyset\vee en\neq\emptyset \vee am\neq\emptyset$}
            \State $action\gets\select$ from $\buds(ct) \cup am\cup (reds_L \cup en)$
            \If{$action$ is a configuration $(s,p)$ from $\buds(ct)$}
                \State $ct\gets\grow(ct,s,p,t)$
                \State $reds_L\gets reds_L\cup\matches_L(s,p,t)$
                \State $am\gets am\cup\matches_{NL}(s,p,t)$
            \ElsIf{$action$ is an ambiguous redex $(\ell\to r)@p$ from $am$}
                \State $am\gets am\setminus\{(\ell\to r)@p\}$
                \If{$t|_p$ is consistent with $\pi(\ell)$}
                    \State $en \gets en\cup\{(\ell\to r)@p\}$
                \Else
                    \State $dis\gets dis\cup\{(\ell\to r)@p\}$
                \EndIf
            \ElsIf{$action$ is a redex $(\ell\to r)@p\in reds_{L}\cup en$}
                \State $reds_L \gets
                    reds_L \setminus \bigcup\{\matches_L(s,q,t)\mid (s,q)\in\nodes(ct[p])\}$
                \State \textsc{Update}$(am,dis,en,ct,t,p)$
                \State $ct\gets \prune(ct,p)$
                \State $t\gets t[(\ell\to r)@p]$
            \EndIf
        \EndWhile
        \Return $t$
    \EndProcedure
    \\
    \Procedure{Update}{$am$, $dis$, $en$,$ct$,$t$,$p$}\Comment{Updates $am$,$dis$ and $en$ after a rewrite}
        \State $en \gets en\setminus \bigcup\{\matches_{NL}(s,q,t)\mid (s,q)\in\nodes(ct[p])\}$
        \State $dis \gets
            dis \setminus \bigcup\{\matches_{NL}(s,q,t)\mid (s,q)\in\nodes(ct[p])\}$
        \State $rem \gets \{
            (\ell\to r)@q \in en\cup dis$
        \State \hspace{3.5em}
            $ \mid \exists P\in\pi(\ell): |P|\geq 2 \wedge \exists q'\in P: q.q'\leq p
        \}$
        \State $en\gets en\setminus rem$
        \State $dis\gets dis\setminus rem$
        \State $am \gets 
            (am \setminus \bigcup\{\matches_{NL}(s,q,t)\mid (s,q)\in\nodes(ct[p])\})\cup rem$
    \EndProcedure
\end{algorithmic}
\end{algorithm}

\begin{example}
We demonstrate Algorithm~\ref{alg:nonlinearrewriter} with the TRS given by the rules
\[
R_1:f(x,x,y)\to y\quad R_2:f(x,y,y)\to x \quad R_3:f(x,y,x)\to x\quad R_4:h(a)\to a
\,,
\]
and the term $t_0=f(a,h(a),h(a))$.
If we match $t_0$ before doing any consistency checks,
we obtain $reds_L=\{R_4@2,R_4@3\}$ and $am=\{R_1@\epsilon,R_2@\epsilon,R_3@\epsilon\}$.
Consistency checking results in $am=\emptyset$,
$dis=\{R_1@\epsilon,R_3@\epsilon\}$, and $en=\{R_2@\epsilon\}$.

Suppose that we apply the redex $R_4@2$ to get $t_1=f(a,a,h(a))$.
The update function reorganizes the non-linear pre-matches.
The disabled pre-match $R_1@\epsilon$ and the enabled pre-match $R_2@\epsilon$
are moved to $am$,
since their consistency status depends on a subterm that changed due to the rewrite step.
The disabled pre-match $R_3@\epsilon$ remains disabled.
\end{example}

\subsection{Correctness}
Partial correctness of Algorithm~\ref{alg:nonlinearrewriter}
can be shown similarly to how we proved partial correctness of Algorithm~\ref{alg:pruningrewriter}.
The most notable difference is the administration that the update function facilitates.
The following proposition states that (in)consistency is preserved for all some pre-matches
under a certain condition.

\begin{proposition}\label{prop:consistencypreservation}
Let $t$ be a term such that $t|_p$ pre-matches $\ell$,
and let $t'=t[u]_{p'}$ for some $u$ and some $p'\geq p$.
If $t'|_p$ pre-matches $\ell$,
and there is no $P\in\pi(\ell)$ and no position $q'\in P$ with $|P|\geq 2$ and $p'\leq q'$,
then
$t|_p$ is consistent with $\pi(\ell)$, if, and only if, $t'|_p$ is consistent with $\pi(\ell)$
\end{proposition}

\begin{lemma}
Let $t_0$ be a ground term.
The invariant given by the conjunction of
\begin{enumerate}
\item $t_0\rr^* t$
\item $ct\sqsubseteq\completed(t)$
\item $reds_L = \bigcup\{\matches_L(s,p,t)\mid (s,p)\in\confs(ct)\}$
\item $am\cup dis\cup en = \bigcup\{\matches_{NL}(s,p,t)\mid (s,p)\in\confs(ct)\}$
\item $reds_L$, $am$, $dis$, and $en$ are disjoint
\item $\forall (\ell\to r)@p\in en: \text{$t|_p$ is consistent with $\pi(\ell)$}$
\item $\forall (\ell\to r)@p\in dis: \text{$t|_p$ is not consistent with $\pi(\ell)$}$
\end{enumerate}
is initialized and maintained by Algorithm~\ref{alg:nonlinearrewriter}.
\end{lemma}
\begin{proof}
We only highlight the differences with the proof of Lemma~\ref{lem:pruninginvariant}.
For Item~1,
since Algorithm~\ref{alg:nonlinearrewriter} can also pick non-linear redexes from $en$,
we need Lemma~\ref{prop:nlmatching} and Item 6 of the invariant
to show that $t_0\rr^*t$ is maintained.
Items~3 and 4 are shown similarly to Lemma~\ref{lem:pruninginvariant},
but the proofs are more cumbersome.
Item~5 is a standard exercise in set theory and program correctness.
Items~6 and 7 are preserved due to Proposition~\ref{prop:consistencypreservation}.
\end{proof}

\begin{corollary}
If \textsc{NormalizeNonLinear}($t_0$, $\select$) returns $t$ then $t$ is a normal form of $t_0$.
\end{corollary}

\section{Implementation}\label{sec:implementation}
As a proof of concept, we have implemented the set automaton construction
and an instance of Algorithm~\ref{alg:nonlinearrewriter} in the Rust programming language.
We have dubbed the implementation SABRE: Set Automaton Based Rewrite Engine.
For comparison we have also implemented a Simple Innermost Rewrite Engine (SIRE)
using adaptive pattern matching automata from \cite{sekar:adaptive}.
Both rewriters use maximal subterm sharing, just like mCRL2.
This allows compact representation of terms and cheap term duplication.
Moreover, checking the equality of subterms is a constant time operation. 
The binaries, along with instructions on how to replicate our benchmarking results are available
on Zenodo\footnote{\url{https://doi.org/10.5281/zenodo.7197261}}.

\paragraph{Configuration stack.}
SABRE uses a stack to explore the configuration tree depth-first.
An element on the stack consists of a configuration
and a pointer to the remaining set automaton edges in a hypertransition.
Another optimisation is that a configuration does not store the absolute position and
the top-level term.
Instead it stores the subterm and the positional difference compared to the configuration one layer higher in the stack.
This positional difference stems from a transition in the set automaton; the tool stores a pointer to this position.
These measures ensure that we can store a configuration with a fixed amount of memory,
i.e. no heap allocations are needed.
Moreover, with an additional trick, the prune operation described in Section~\ref{sec:pruningrewriter}
can be done in constant time.
This trick in the set automaton construction involves keeping track of the amount of function symbols
that have been inspected since the creation, for every match goal. 

\paragraph{Outermost matching.}
Recall that set automata enable many traversals of the configuration tree.
We implemented a different dependency relation to construct the set automaton
so that the aforementioned depth-first traversal yields outermost matches first.
On match goals $g_1=mo_1\to\ell_1@p_1$ and $g_2=mo_2\to\ell_2@p_2$ we use the \emph{outermost preserving dependency
relation} $R_{op}$ defined by
$g_1\mathrel{R_{op}} g_2$ if, and only if, $p_1\leq p_2\vee p_2\leq p_1$.
This relation is coarser than the relation $R_{dep}$ described in Section~\ref{sec:construction},
and intuitively maintains look-ahead on patterns until they are not guaranteed to be outermost.
Additionally it preserves all the correctness properties of pattern matching and has the additional advantage that
a depth-first traversal of the configuration tree yields outermost matches first,
unless the first discovered redex is a subterm of an outermost match that may be discovered at a later point.
For example, given the patterns $f(g(h(x)))$ and $g(x)$,
we will always discover $g(x)$ first.
A drawback of this relation is that it yields a bigger automaton.
We refrained from presenting this relation thus far
since the problem is subtle and yields set automata that are too complex to explain.
An in-depth study of these dependency relations is left as future work.

\paragraph{Strategy.} SABRE has a baked-in, mostly outermost strategy.
In principle SABRE applies a rewrite rule at the moment a match is found,
with three exceptions that we will explore.
Firstly, the tool has special treatment of \emph{duplicating} rewrite rules,
i.e. rules for which application duplicates subterms, e.g. $f(x) \to g(x,x)$.
Such rewrite rules are parked on a separate strategy stack,
forming a data structure to remember matching rewrite rules that are not applied upon discovery.
The rule is applied at the moment the entire subconfiguration tree below the node
where the duplicating rule was found has been explored.
This means that when the rule is applied, its subterms (that are duplicated)
are guaranteed to be in normal form.
Secondly, non-linear rewrite rules are not always performed in outermost fashion.
Upon discovery, equivalence of subterms is checked.
If the equivalences do not (yet) hold the match is parked on the strategy stack
until the rest of the configuration subtree has been explored.
Finally, our implementation also supports conditional rewrite rules which is needed
as they are part of the chosen benchmarks.
The implementation is ahead of the theory in this regard.
Conditional rewrite rules are always first put on the strategy stack.
The conditions are evaluated once the entire subconfiguration tree of the conditional pre-match is explored. 

\paragraph{Benchmark selection.}
We have turned to a set of benchmarks used in the recurring Rewriting Engine Competition (REC).
This competition has its own format for TRSs: the REC format.
Translations exist from this REC format to a multitude of popular rewrite engines.
The REC format does not support non-linear rewrite rules,
as not all tools in the competitions support it.
We use the benchmarks available on the webpage\footnote{\url{https://sourcesup.renater.fr/scm/viewvc.php/rec/2019-CONVECS/}}, which contains benchmarks in preparation for a next REC. Some TRSs are excluded.
Some rewrite specifications include a ``META'' section containing code to generate a part of the specification,
we do not support this feature yet.
We have also excluded the benchmarks ``langton6'' and ``langton7'' as they implicitly assume
a priority order on the rewrite rules based on the order in which they are listed,
which we intentionally do not support (and which is not mentioned in \cite{garavel:rec}).
In total we use 74 benchmarks,
which are listed in the Zenodo repository.

\paragraph{Tool selection.}
We decided to benchmark SABRE, SIRE, mCRL2, Maude and Haskell.
There are two rewriters in the mCRL2 toolset:
jitty (just-in-time) and jittyc (compiling jitty).
The jittyc rewrite engine generates and compiles C++ code specific to the TRS.
Haskell and Maude were selected because they won the last competition (REC 2018);
Haskell was the fastest rewrite engine overall,
and Maude was the fastest rewriter among the algebraic languages.

\paragraph{Benchmark setup.}
There are multiple decisions to be made in measuring rewriting performance.
We have opted to split preprocessing time and rewriting time wherever possible.
The context of the development of SABRE is model checking, where we typically need to rewrite millions or billions of terms,
making preprocessing time less relevant.
For jitty and Maude we were not able to split the preprocessing and running times.
They therefore have a (slight) handicap in the running time benchmark.
As not every tool is able to solve every benchmark in a reasonable time frame we use a timeout of 5 minutes,
just like in the REC.
As performance metrics we use both the number of problems solved and the total time spent.
All benchmarks are performed on a 2018 MacBook pro equipped with an 2.6GHz i7 processor and 32GB of RAM.

\begin{table}[htb]
\centering
\begin{tabular}{|l|r|r|r|r|r|r|}
\hline
\textbf{Benchmark} & \textbf{SABRE} & \textbf{SIRE} & \textbf{jitty} & \textbf{jittyc} & \textbf{Maude} & \textbf{Haskell}\\
\hline
        benchsym22 & 9.45 & 5.51 & 31.27 & 0.67 & 3.76 & 0.54 \\ \hline
        benchtree10 & 0.70 & DNF & DNF & 41.00 & 0.51 & 0.22 \\ \hline
        benchtree20 & 11.53 & DNF & DNF & DNF & 6.29 & 1.24 \\ \hline
        benchtree22 & 50.28 & DNF & DNF & DNF & 23.20 & 6.27 \\ \hline
        binarysearch & 11.97 & 12.52 & 61.13 & 1.58 & 24.50 & 1.93 \\ \hline
        bubblesort1000 & 47.18 & 5.70 & 92.17 & 0.75 & 3.46 & 0.49 \\ \hline
        evalsym & 33.71 & 54.43 & DNF & 8.44 & 148.85 & 8.58 \\ \hline
        evaltree & 1.33 & 2.22 & 32.67 & 32.05 & 16.13 & 1.40 \\ \hline
        fib32 & 9.33 & 5.45 & 30.66 & 1.44 & 12.67 & 0.99 \\ \hline
        maa & 136.44 & 27.29 & 42.59 & 1.80 & 22.32 & 0.95 \\ \hline
        mergesort100 & 0.80 & DNF & DNF & DNF & 0.47 & 0.24 \\ \hline
        mergesort1000 & 0.35 & DNF & DNF & DNF & 0.29 & 0.12 \\ \hline
        quicksort100 & 0.34 & DNF & DNF & DNF & 0.55 & 0.23 \\ \hline
        quicksort1000 & 76.11 & DNF & DNF & DNF & 7.23 & 2.30 \\ \hline
        revnat10000 & 56.37 & 31.51 & 33.33 & 2.97 & 14.18 & 6.69 \\ \hline
        sieve2000 & 20.78 & 5.31 & 93.20 & 0.24 & 2.87 & 0.38 \\ \hline
        sieve10000 & DNF & DNF & DNF & 13.54 & 221.90 & 26.35 \\ \hline
        \textbf{Total successes} &  73  &  66  &  65  &  68  &  74  &  74  \\ \hline
        \textbf{Total failures }& 1 & 8 & 9 & 6 & 0 & 0 \\ \hline
        \textbf{Percentage} & 98.7\% & 89.2\% & 87.8\% & 91.9\% & 100.0\% & 100.0\% \\ \hline
        \textbf{Total time} & 813 & 2576 & 3257 & 1909 & 553 & 73 \\ \hline
\end{tabular}
\caption{Running times in seconds to solve a given benchmark,
preprocessing omitted where possible.
DNF = Did Not Finish.
DNF is counted as 300 seconds in the computation of the total time spent.}
\label{tab:benchmarks-running-time}
\end{table}

\paragraph{Benchmark running time results.}
Table \ref{tab:benchmarks-running-time} shows the time it takes to solve the REC problems.
There is not enough space to show the results of all benchmarks.
The table lists all benchmarks where at least one of the tools took more than 30 seconds to finish.
SABRE comfortably beats SIRE, jitty and jittyc, but is beaten by Maude and Haskell.
SABRE does beat Maude in 4 benchmarks by some margin: `binarysearch', `evalexpr', `evalsym' and `evaltree'.
In some benchmarks it appears that the innermost strategies of SIRE, jitty and jittyc are inferior
by a large margin: `benchtree', `mergesort' and `quicksort'.

\paragraph{Preprocessing evaluation.}
As mentioned,
the preprocessing time is of secondary value to us.
Nevertheless, it is important that the preprocessing time is not excessive.
Table \ref{tab:benchmarks-preprocessing} presents statistics on the preprocessing time and the size of the set automaton.
Only the benchmarks with at least 50 rewrite rules are presented in the table.
Constructing the set automaton for the `maa' benchmark is rather expensive.
The determining factor for the cost of set automaton construction seems to be the number of transitions,
which is determined by the equation $\text{\#transitions} = \text{\#symbols} \times \text{\#states}$.
The number of states does not directly correlate to the number of symbols or rewrite rules.
Currently we do not know a method to compute the size of a set automaton
without generating it.
The determining factor seems to be the pattern depth and overlap in the pattern set.

\begin{table}[htb]
\centering
\begin{tabular}{|p{2.5cm}|r|r|r|r|r|r|}
\hline
\textbf{Benchmark} & \thead{\textbf{\#symbols}/ \\\textbf{\#rules}} & \thead{\textbf{\#states}/\\\textbf{\#transitions}} &\textbf{SABRE} & \textbf{SIRE} & \textbf{jittyc} & \textbf{Haskell}\\
\hline
        benchexpr10 & 92/155 & 47/4,324 & 0.13 & 0.50 & 9.81 & 0.81 \\ \hline
        binarysearch & 43/86& 1,282/55,126 & 3.50 & 0.30 & 10.36 & 0.78 \\ \hline
        evalexpr & 27/70 & 24/648 & 0.23 & 0.20 & 8.48 & 0.75 \\ \hline
        fib32 & 33/69 & 1,036/34,188 & 1.76 & 0.30 & 9.43 & 0.79 \\ \hline
        fibfree & 27/232& 101/2,727 & 0.22 & 0.30 & 9.16 & 0.87 \\ \hline
        maa & 708/750 & 831/588,348& 389.69 & 0.73 & 92.94 & 2.32 \\ \hline
        \textbf{Total time} & ~ & ~ & 406.402 & 5.43 & 697.24 & 55.48 \\ \hline
\end{tabular}
\caption{Statistics on the preprocessing time needed by the tools. }
\label{tab:benchmarks-preprocessing}
\end{table}


\paragraph{Rewrite path comparison.}
An interesting question is why certain tools are faster than others on certain benchmarks:
is the rewriting path shorter or is the number of rewrites per second higher?
Table \ref{tab:benchmarks-rewrite-steps} shows the number of rewrite steps needed by SABRE, SIRE and Maude,
for the same benchmarks listed in Table \ref{tab:benchmarks-running-time}.
The other tools do not report the number of rewrite steps.
The overall picture is that SABRE needs fewer rewrite steps than SIRE and Maude, with some exceptions.
SABRE needs fewer symbol inspections per rewrite step,
which is expected as the set automaton remembers all the symbols it has inspected whilst matching.
SABRE is comparatively slow in terms of rewrite steps per second. 
\begin{table}[htb]
\centering
\begin{tabular}{|p{4.5cm}|r|r|r|}
\hline
\textbf{Benchmark} & \textbf{SABRE} & \textbf{SIRE} & \textbf{Maude} \\
\hline
                benchsym22 &  90,303,766  &  90,303,766  &  90,303,764  \\ \hline
                benchtree10 &  47,213  & DNF &  26,717  \\ \hline
                benchtree20 &  72,356,522  & DNF &  25,170,576  \\ \hline
                benchtree22 &  310,384,073  & DNF &  100,668,845  \\ \hline
                binarysearch &  107,501,220  &  265,511,947  &  310,310,897  \\ \hline
                bubblesort1000 &  167,670,646  &  167,670,646  &  168,171,144  \\ \hline
                evalsym &  361,204,694  &  1,438,834,118  &  1,664,833,899  \\ \hline
                evaltree &  15,208,759  &  62,407,963  &  104,614,038  \\ \hline
                fib32 &  84,116,962  &  108,788,784  &  137,416,567  \\ \hline
                maa &  392,119,718  &  441,111,164  & >395,558,243$^+$ \\ \hline
                mergesort100 &  28,278  & DNF &  43,395  \\ \hline
                mergesort1000 &  4,199,124  & DNF &  6,454,547  \\ \hline
                quicksort10 &  695  &  40,645  & 959 \\ \hline
                quicksort100 &  207,982  & DNF & 384,630 \\ \hline
                quicksort1000 &  170,679,654  & DNF &  338,346,151  \\ \hline
                revnat10000 &  50,046,171  &  50,046,171  &  50,046,168  \\ \hline
                sieve2000 &  63,273,997  &  124,110,525  &  124,308,194  \\ \hline
                sieve10000 & DNF & DNF &  9,898,601,115  \\ \hline
                \textbf{Total rewrite steps$^*$} &  1,579,878,114 & 3,149,524,528 & > 3,3 $\times 10^9$\\ \hline
                \textbf{Total symbol inspections$^*$} &   4,001,712,359 &  9,286,297,430  & unknown\\ \hline
                \textbf{Symbol inspections/ \newline rewrite steps$^*$} & 2.53& 2.95 & unknown\\ \hline
                \textbf{Rewrite steps/second$^*$} & 4,236,813& 17,866,702 & >11,513,173,21\\ \hline
\end{tabular}

\caption{Rewrite steps per benchmark.
DNF = Did Not Finish.
$^*$For these metrics we only selected the benchmarks for which SABRE, SIRE and Maude finished in time.
$^+$ The `maa' benchmark consists of many terms that need to be rewritten,
of which Maude reports the number of rewrite steps per term individually.
The figure in the table represents the number of rewrite steps it takes to rewrite the last term,
which dominates the benchmark.}
\label{tab:benchmarks-rewrite-steps}
\end{table}

\paragraph{Bottleneck SABRE.}
Profiling SABRE revealed that its bottleneck is constructing terms using the maximally shared term storage.
A lookup hash table is used to check if a term already exists in the storage.
For each rewrite step several lookups in this hash table are needed.
An optimisation implemented in SIRE and mCRL2s rewriters is to not store
the term that is currently being rewritten in the maximally shared storage.
Suppose that you have the following rewrite rule for subtraction $sub(s(x),s(y)) \to sub(x,y)$.
This rewrite rule will be executed in a large sequence if $x$ and $y$ represent large numbers.
SABRE will call the term storage for every intermediate result.
SIRE, jitty and jittyc will locally store the head symbol $sub$ and the list of subterms.
Applying this rewrite rule can be done by simply updating the list of subterms locally,
avoiding the lookup table.
It is likely that something similar can be done for SABRE,
though the control flow is more difficult than with simple innermost rewriters.

\section{Conclusion, discussion and future work}\label{sec:conclusion}
In this paper we established a formal framework for rewriting based on set automaton matching.
We discussed a general rewrite procedure and proved its correctness,
and provided an extension that allows non-linear TRSs.
We also discussed a specific instance of the general rewrite procedure that
runs an outermost rewrite strategy by saving the configuration tree as a stack.
Our benchmarks show that this method is competitive among existing tools.
There are many optimizations left to explore, some of which require more theoretical exploration
and some of which are more implementation-based.

\paragraph{Evaluation benchmarks.}
Besides being fast, rewriters must obviously also rewrite correctly.
We have checked the output of every REC benchmark
and compared it to the output of other rewriting engines,
revealing that the langton benchmarks are not suitable as they assume priority.
For 50 `quick' benchmarks we have validated and stored the result; they form a suite of regression tests.

For any benchmark suite it is the question how representative it is of real world specifications.
We have opted to use the rather large and well known benchmark suite from the rewriting engine competitions.
Most benchmarks do not have a lot of rewrite rules,
making construction of the set automaton relatively easy.
For comparison, there are 478 rewrite rules in the standard datatype specification of mCRL2,
and industrial models can contain thousands of rewrite rules.
An omission in the REC benchmarks is the lack of non-linear rewrite rules,
we have thus not benchmarked SABRE in the presence of non-linear rules.

\paragraph{Future work.}
In the future we want to expand the set automaton rewriting technique
to minimize the information lost upon performing a rewrite step.
For instance,
the structure of the right-hand sides is not used in our procedures.
A theoretical optimization is to precompile the configuration tree fragments that arise from the right-hand sides,
ensuring that we do not have to match the function symbols of the right-hand side after subterm replacement.

Building upon this optimization, this new configuration tree fragment can yield redexes straight away.
This results in another rewrite, and another precompiled fragment that can be attached to an even higher pruning point.
If such a rewrite sequence is long, then it gives rise to the following algorithmic problem.
Suppose that
there exists a sequence $\{t_i\}_{1\leq i\leq n}$,
such that
$t_0\step{R_1@p_1} t_1\step{R_2@p_2}\dots\step{R_n@p_n}t_n$.
How do we compute $t_n$ efficiently,
without necessarily computing every intermediate $t_i$?
Seeing as a lot of time is spent executing rewrite steps,
including many queries to the maximally shared term storage,
it might be beneficial to investigate this problem.


Our benchmarks revealed that there is a great performance difference between
rewriting strategies.
Rewriting strategies have been subjected to extensive study,
both of practical and theoretical nature
\cite{odonnell:computing,ogata:estrategy,pol:justintime,terese,visser:stratego}.
Strategies are usually defined as functions that yield a reduction path for a term,
or as subsets of the rewrite relation $\rr^*$.
The works that focus on strategies typically abstract from the fact that pattern matching is
a non-trivial algorithmic problem.
We intend to integrate the existing work on strategies with our methods.
A requirement is that deciding which step to take next can be done efficiently
in unison with matching using the set automaton.

The proof of concept implementation of the set automaton is competitive with
existing tools but is beaten by Maude, the fastest rewriter in its class.
Haskell is still an order of magnitude faster but imposes different restrictions on the rewrite rules.
Constructing the set automaton can be slow for larger rewrite systems.
Compared to innermost rewriters (such as jittyc) SABRE is much faster when
the outermost rewriting path is shorter.
When the innermost and outermost paths are equal in length it comes down to raw rewriting performance,
where SABRE is currently beaten.
We reckon there is still room to optimize the raw rewriting performance,
especially considering it is still an experimental prototype.
More specifically, as suggested in Section \ref{sec:implementation},
we could minimize the number of lookups in the maximally shared term storage by first rewriting terms locally.